\documentclass[11pt]{article}

\usepackage{amsthm}
\usepackage{mathtools}
\usepackage{amsfonts}

\usepackage{booktabs}
\usepackage[a4paper, total={8.5in, 11in},margin=1in]{geometry}
 
\usepackage{comment}
\usepackage[labelfont=bf]{caption}

\usepackage{dsfont}

\usepackage[ruled,vlined]{algorithm2e}

\usepackage[colorlinks=true,linkcolor=red,citecolor=blue]{hyperref}
\usepackage{cleveref}
\DeclarePairedDelimiterX{\infdivx}[2]{(}{)}{%
  #1\;\delimsize\|\;#2%
}

\newtheorem{theorem}{Theorem}[section]
\newtheorem{definition}[theorem]{Definition}

\newtheorem{lemma}[theorem]{Lemma}

\usepackage[round]{natbib}

\DeclareMathOperator*{\E}{\mathbb{E}}

\DeclareMathOperator{\poly}{poly}
\DeclareMathOperator{\sign}{sign}
\DeclareMathOperator{\err}{err}
\DeclareMathOperator{\opt}{OPT}
\DeclareMathOperator{\noisy_oracle}{EX^{noisy}}
\DeclareMathOperator{\noiseless_oracle}{EX}
\DeclareMathOperator{\oracle_eta}{EX^{\eta}}
\DeclareMathOperator{\stat}{STAT}

\DeclareMathOperator{\pr}{\mathbf{Pr}}

\DeclareMathOperator{\algocsq}{\textsc{RobustLearningViaCSQs}}
\DeclareMathOperator{\algosq}{\textsc{RobustLearningViaSQs}}

\author{
  Ioannis Anagnostides\\[-2mm]
  National Technical University of Athens\\[-2mm]
  \texttt{ioannis.anagnostides@gmail.com}
  \and
  Themis Gouleakis\\[-2mm]
  Max Planck Institute for Informatics\\[-2mm]
  \texttt{tgouleak@mpi-inf.mpg.de}
  \and 
  Ali Marashian\\[-2mm]
  Sharif University of Technology\\[-2mm]
  \texttt{marashian@ce.sharif.edu}
}

\date{}                     

\title{Robust Learning under Strong Noise via SQs}

\begin{document}

\maketitle
\pagenumbering{gobble}

\begin{abstract}
    This work provides several new insights on the robustness of Kearns' statistical query framework against challenging label-noise models. First, we build on a recent result by \cite{DBLP:journals/corr/abs-2006-04787} that showed noise tolerance of distribution-independently evolvable concept classes under Massart noise. Specifically, we extend their characterization to more general noise models, including the Tsybakov model which considerably generalizes the Massart condition by allowing the flipping probability to be arbitrarily close to $\frac{1}{2}$ for a subset of the domain. As a corollary, we employ an evolutionary algorithm by \cite{DBLP:conf/colt/KanadeVV10} to obtain the first polynomial time algorithm with arbitrarily small excess error for learning linear threshold functions over any spherically symmetric distribution in the presence of spherically symmetric Tsybakov noise. Moreover, we posit access to a stronger oracle, in which for every labeled example we additionally obtain its flipping probability. In this model, we show that every SQ learnable class admits an efficient learning algorithm with $\opt + \epsilon$ misclassification error for a broad class of noise models. This setting substantially generalizes the widely-studied problem of classification under RCN with known noise rate, and corresponds to a non-convex optimization problem even when the noise function -- i.e. the flipping probabilities of all points -- is known in advance.
\end{abstract}

\clearpage

\pagenumbering{arabic}

\section{Introduction}

Label noise is a critical impediment in machine learning as it may dramatically reduce the accuracy of the classifier and augment the computational and sample requirements of the learning algorithm. Naturally, designing efficient and noise tolerant paradigms has been a central endeavor from the inception of machine learning with Rosenblatt's celebrated perceptron algorithm (\cite{Rosenblatt1958ThePA}). Indeed, a vast body of work has been devoted to tackling different types of label noise. As it turns out, the guarantees we can hope for crucially depend on the underlying noise model. In the agnostic model (\cite{HAUSSLER199278,DBLP:conf/colt/KearnsSS92}) -- where an adversary is allowed to corrupt some fraction of the labels arbitrarily --  even weak learning is known to be NP-hard (\cite{4031391,4031389,DBLP:conf/stoc/Daniely16}), while the best results require additional assumptions on the marginal distribution over the instance space, and obtain much weaker multiplicative approximations (\cite{DBLP:conf/colt/Daniely15,DBLP:conf/stoc/AwasthiBL14}). On the other hand, in the random classification noise (henceforth RCN) of \cite{DBLP:journals/ml/AngluinL87} -- where the label of each example is flipped independently with some \emph{fixed} probability $\gamma < \frac{1}{2}$ -- strong positive results have been established, commencing from \cite{10.1145/180139.181176} and \cite{DBLP:conf/focs/BlumFKV96}. Typical approaches to learn in the presence of noise include empirical risk minimization (ERM) (\cite{DBLP:journals/tcyb/ManwaniS13}) with a smooth and convex surrogate of the $0-1$ loss (\cite{10.2307/30047445}), boosting techniques (\cite{friedman2000,DBLP:conf/colt/Freund99}), variants of the perceptron algorithm (\cite{DBLP:conf/nips/LiL99,DBLP:journals/jmlr/KhardonW07}), and SVMs (\cite{DBLP:conf/interspeech/GanapathirajuP00,DBLP:journals/prl/LinW04}). 

Moreover, a powerful technique for designing noise tolerant algorithms was introduced by \cite{10.1145/293347.293351} in the form of the \emph{statistical query} (SQ) framework, a natural restriction to Valiant's PAC learning model (\cite{DBLP:conf/stoc/Valiant84}) in which the learner employs ''large'' samples, instead of properties of specific individual examples. Importantly, Kearns demonstrated that any procedure based on statistical queries can be automatically converted to a learning algorithm robust to the RCN model, with rate smaller than the information-theoretic barrier of $\frac{1}{2}$. In fact, virtually all PAC learnable concept classes are also learnable with access to statistical queries; a notable exception is the class of parities that includes all functions equal to an XOR of some subset of Boolean variables (\cite{10.1145/792538.792543}). As a result, most RCN-tolerant PAC learning algorithms were either derived from the SQ model, or can be easily cast into it. Our work follows this long line of research and strengthens prior results along several lines, extending the robustness of the SQ framework against more challenging noise models.  

In the first part of our work we consider the standard noisy oracle model, defined in the following generic form:

\begin{definition}[Noisy Oracle]
    \label{definition:noisy_oracle}
Let $f$ be an unknown target function in a family of Boolean functions $\mathcal{C}$ over $\mathcal{X} \subseteq \mathbb{R}^d$, $\mathcal{D}_{\mathbf{x}}$ an arbitrary distribution over $\mathcal{X}$, and $\eta : \mathcal{X} \mapsto \left[0, \frac{1}{2} \right)$ an unknown function. A noisy oracle $\noisy_oracle(f, \mathcal{D}_{\mathbf{x}}, \eta)$ returns a labeled example $(\mathbf{x}, y)$, where $\mathbf{x} \sim \mathcal{D}_{\mathbf{x}}$, $y = f(\mathbf{x})$ with probability $1 - \eta(\mathbf{x})$ and $y = - f(\mathbf{x})$ with probability $\eta(\mathbf{x})$. We let $\mathcal{D}$ represent the joint distribution on $(\mathbf{x},y)$ generated by the above oracle.
\end{definition}

Of course, the known guarantees crucially depend on the assumptions we make on the noise function $\eta$. As we previously discussed, the special case of the RCN model -- i.e. $\eta(\mathbf{x}) = \gamma < \frac{1}{2}$ -- is known to be efficiently addressed due to the symmetry of the underlying noise. A much more realistic and widely-studied condition that has received considerable attention in computational learning theory in recent years (\cite{DBLP:conf/colt/AwasthiBHZ16,pmlr-v65-zhang17b,NIPS2017_6706,DBLP:journals/corr/abs-2002-04840,DBLP:conf/colt/AwasthiBHU15,DBLP:journals/corr/abs-2002-05632,DBLP:conf/nips/DiakonikolasGT19,DBLP:journals/corr/abs-2006-04787}) is the so-called Massart or \emph{bounded} model (\cite{massart2006}), where $\eta(\mathbf{x}) \leq \gamma$ for some parameter $\gamma \in \left[0, \frac{1}{2}\right)$. In this paper, we are tackling a substantially more general model than the Massart, defined as follows.

\begin{definition}[Tsybakov noise]
    A noise function $\eta$ satisfies the $(\alpha, A, t_0)$-Tsybakov\footnote{Sometimes referred to as the Mammen-Tsybakov condition} condition with $\alpha \in (0,1)$, $A > 0$, and $t_0 \in \left( 0, \frac{1}{2} \right]$ if 
    
    \begin{equation}
        \pr_{\mathbf{x} \sim \mathcal{D}_{\mathbf{x}}}\left[\eta(\mathbf{x}) \geq \frac{1}{2} - t \right] \leq A t^{\frac{\alpha}{1 - \alpha}},
    \end{equation}
    for all $t \in [0, t_0]$.
\end{definition}

Thus, the Tsybakov condition allows the flipping probability to be arbitrarily close to $\frac{1}{2}$ for a subset of the domain, and as such, it strictly generalizes the Massart condition. Indeed, for $\alpha \to 1$ it follows that $t^{\frac{\alpha}{1 - \alpha}} \to 0$, and hence, the Tsybakov condition yields a $\gamma$-Massart noise with $\gamma = \frac{1}{2} - t_0$. This particular noise model was introduced by \cite{mammen1999} in a slightly stronger form, and it was subsequently refined by \cite{10.2307/3448504}. However, although the information-theoretic aspects of the Tsybakov noise have been well-understood in statistics (\cite{boucheron_bousquet_lugosi_2005,10.2307/30047445,10.5555/945365.964298,morel2007concentration,koltchinskii2006,10.5555/2789272.2912111}), developing computationally efficient algorithms has remained elusive and a notable open problem in computational learning theory. Importantly, the techniques employed in the Massart model inherently fail and new algorithmic ideas are required. In this context, we build on a recent result of \cite{DBLP:journals/corr/abs-2006-04787} which showed that Valiant's notion of \emph{evolvability} (\cite{DBLP:journals/jacm/Valiant09}) implies efficient learnability in the presence of Massart noise. Our first contribution is to extend their characterization to a broader class of noise models, and as a corollary, we establish a polynomial time algorithm with arbitrarily small excess error in the presence of spherically symmetric Tsybakov noise.

In the second part of our study we posit a stronger oracle. Specifically, we assume that for every labeled example $(\mathbf{x}, y)$ drawn from $\mathcal{D}$ we additionally obtain its corruption rate $\eta(\mathbf{x})$. This model strongly generalizes the widely-studied RCN with known rate\footnote{In RCN the assumption of known noise rate can be easily removed through simple techniques; e.g., see \cite{10.5555/46710}}, and is motivated in a number of practical applications. Indeed, label-noise typically reflects a measure of confidence or uncertainty of the experts with respect to the given instance (\cite{DBLP:journals/tnn/FrenayV14}), and it is natural to assume that the expert can provide a quantifiable measure of that uncertainty. In some cases multiple experts may be employed in order to determine the label in a subjective task (\cite{10.2307/2346806}), and the error rate can be approximated through the disagreement among the experts; typical scenarios include medical applications (\cite{10.1093/bioinformatics/btl346}) and image data analysis (\cite{10.5555/2998687.2998822,10.1016/0167-8655(96)00105-5}). Moreover, label noise may be also caused from communication errors (\cite{DBLP:journals/jair/BrodleyF99}) and knowing the reliability of the different channels of communication would provide an estimate of the noise rate for every observed labeled example. Finally, stochastic label noise can be intentionally introduced in order to protect the agents' privacy; in such cases, the noise may be fully specified (\cite{10.2307/1403910}) and the crucial question is how robust is the system to potential leakage of the noise function to adversarial and potentially malicious parties. 

Our key contribution in this model is to show that with access to such an oracle, every concept class efficiently learnable with statistical queries admits a polynomial time learning algorithm with arbitrarily small excess error for a broad class of noise models, including Massart and Tsybakov\footnote{Although it should be noted that the Tsybakov condition is much more benign with such an oracle}. Our argument extends Kearns' celebrated simulation of statistical queries in the presence of RCN, and is of particular importance in light of lower bounds against convex surrogates in the presence of Massart noise, applicable even when the noise function is known (\cite{DBLP:conf/nips/DiakonikolasGT19}). 

\subsection{Related Work}

Developing efficient and noise tolerant algorithms via the statistical query framework has been an area of prolific research from its inception. Indeed, the main motivation of Kearns' (\cite{10.1145/293347.293351}) original formulation was to effectively combat the RCN model. Subsequently, a considerable number of works have pursued analogous guarantees against various noise models. A classic result by \cite{DBLP:conf/colt/Decatur93} established robustness of the statistical query framework against Valiant's \emph{malicious error} model (\cite{10.5555/1625135.1625242}) -- in which the adversary is allowed to distort a fraction of the observed examples, as well as a hybrid model that combined RCN with malicious error (\cite{Decatur1996}). Several results have also been obtained for the so-called \emph{attribute noise} in which every bit of the observed instance $\mathbf{x}$ is flipped with some fixed probability; we refer to the study of \cite{DBLP:conf/colt/DecaturG95} and references therein.  

A more modern line of work has focused on establishing distribution-specific PAC learning algorithms in the presence of Massart noise. More precisely, although the Massart model goes back to \cite{10.5555/93025.93060,10.1145/130385.130433} and \cite{646140} (where it was studied under the name \emph{malicious misclassification noise}), the first polynomial time algorithms began to formulate quite recently due to the assymetric nature of the noise. Indeed, this endeavor was initiated by \cite{DBLP:conf/colt/AwasthiBHU15}, and led to gradual improvements (\cite{DBLP:conf/colt/AwasthiBHZ16,pmlr-v65-zhang17b,NIPS2017_6706,DBLP:journals/corr/abs-2002-04840}) for the fundamental class of LTFs. The state of the art in the distribution-specific setting was recently obtained by \cite{DBLP:journals/corr/abs-2002-05632}, extending the optimality guarantee beyond log-concave and s-concave functions, to a class of well-behaved distributions. Yet, we stress that the aforementioned approaches for $\gamma$-Massart noise inherently fail under the more general Tsybakov condition, given that they need $\Omega(1/(1 - 2\gamma))$ number of samples, and the Tsybakov noise would require $\gamma = \frac{1}{2}$. 
    
On the distribution-independent PAC learning model with Massart noise the main breakthrough was made recently by \cite{DBLP:conf/nips/DiakonikolasGT19}, where the first polynomial time algorithm for LTFs with non-trivial misclassification error was obtained. More precisely, they established a non-proper learner with error $\gamma + \epsilon$, for any $\epsilon > 0$\footnote{This guarantee is -- in general -- information-theoretically sub-optimal and hence, it does not subsume the aforementioned works.}; we remark that this guarantee would fail to yield even a weak learner in the presence of Tsybakov noise. Building on their ideas, a recent work by \cite{DBLP:journals/corr/abs-2006-04787} made several remarkable advancements. Among others, they developed a black-box knowledge distillation procedure which converts any classifier -- potentially non-proper -- to a proper halfspace with equally good performance, while they also proved -- based on a result by \cite{DBLP:journals/jmlr/Feldman11} -- a super-polynomial statistical query lower bound for achieving arbitrarily small excess error in the presence of Massart noise. Naturally, this lower bound is also applicable in the presence of the more general Tsybakov noise.
    
Finally, the first concrete progress in deriving computationally efficient learning algorithms in the presence of Tsybakov noise was made recently by \cite{DBLP:journals/corr/abs-2006-06467}. More precisely, they established an algorithm -- for learning LTFs over well-behaved distributions -- that incurs $\epsilon$ excess misclassification error with sample complexity and running time $d^{\mathcal{O}(\log^2(1/\epsilon)/\alpha^2)}$, where $\alpha$ is the parameter of the Tsybakov condition; thus, their algorithm is quasi-polynomial in $1/\epsilon$, while the dependence on $1/\alpha$ is exponential with respect to dimension of the instance space $d$. Therefore, they left as an open question whether a polynomial time algorithm exists. This problem has been also addressed in a work concurrent to ours by \cite{diakonikolas2020polynomial}, establishing a polynomial time learning algorithm in the presence of Tsybakov noise for a class of well-behaved distributions. Prior to these works, the only known algorithms that succeed in the presence of Tsybakov noise were obtained by a reduction to agnostic learning, and subsequently have a prohibitive complexity. More precisely, the $L_1$-regression algorithm (\cite{DBLP:journals/siamcomp/KalaiKMS08}) requires doubly-exponential running time in order to recover the optimal halfspace even for log-concave distributions. 

\subsection{Our Contributions}

Our work provides several new insights on the robustness of Kearns' statistical query framework with respect to challenging noise models. In the first part we rely on a recent result obtained by \cite{DBLP:journals/corr/abs-2006-04787}; specifically, they showed that any concept class learnable with access exclusively to correlational statistical queries -- or equivalently any evolvable class (see \Cref{section:preliminaries}) -- admits a polynomial time learning algorithm with arbitrarily small excess error in the presence of Massart noise. Our first contribution is to strengthen their argument in the presence of more general noise models. More precisely, we identify a natural measure of the intensity of the noise -- measuring the average proximity of the flipping probability to the barrier of $1/2$ -- which we refer to as the \emph{magnitude} of the noise with respect to the underlying distribution. We observe that this parameter allows for a unifying treatment of a broad class of noise models, and in particular our main focus lies on the Tsybakov condition, a notoriously hard model which allows the flipping probability to be arbitrarily close to $1/2$; as such, it substantially generalizes the Massart condition. In light of this, although the Tsybakov noise model has received considerable attention in statistics from an information-theoretic standpoint, establishing efficient learning algorithms has remained elusive and a notable open problem in computational learning theory. 

In this context, we show that when the magnitude of the noise is polynomially-bounded, the algorithm proposed by \cite{DBLP:journals/corr/abs-2006-04787} (see \Cref{algorithm:CSQ_reduction}) efficiently obtains a hypothesis with arbitrarily small excess error for any evolvable class. Informally, we establish the following theorem:

\begin{theorem}
    \label{theorem:CSQ_learnability}
    Consider any CSQ learnable concept class, and any distribution $\mathcal{D}$ corrupted with label-noise of polynomially-bounded magnitude. Then, there exists a polynomial time learning algorithm that takes as input samples from $\noisy_oracle$, and returns with high probability a hypothesis $h$ such that for any $\epsilon > 0$, $\err_{\mathcal{D}}(h) = \opt + \epsilon$, where $\opt$ is the misclassification error of the optimal classifier.
\end{theorem}

For a formal statement of this theorem we refer to \Cref{theorem:CSQ_learnability}. We combine this result with a sharp upper-bound we derive for the magnitude of the Tsybakov noise (\Cref{lemma:Tsybakov-magnitude}) to obtain the first polynomial time algorithm with arbitrarily small excess error for non-trivial concept classes in the presence of noise with polynomially bounded magnitude. Specifically, our guarantee is applicable for the fundamental class of LTFs, and for any spherically symmetric distribution (e.g. uniform distribution on the unit-sphere) and spherically symmetric Tsybakov noise (\Cref{theorem:Tsybakov}), and follows directly from an evolutionary algorithm developed by \cite{DBLP:conf/colt/KanadeVV10}.

The second part of our work focuses on a more benign noise model. Specifically, we assume access to a stronger oracle $\oracle_eta$ that along with a labeled example $(\mathbf{x}, y)$ returns the flipping probability $\eta(\mathbf{x})$ for this particular point. We stress that this model strictly generalizes the widely-studied and non-trivial model of the random classification noise (RCN) with a known noise rate $\eta$, in which the adversary flips every label with \emph{fixed} probability $\eta$. Kearns' pivotal work (\cite{10.1145/293347.293351}) established that the statistical query framework is robust to the RCN model in the sense that every statistical query can be simulated with high probability even with access to a noisy-RCN oracle (see \Cref{lemma:Kearns-simulation}). Indeed, every concept class efficiently learnable with statistical queries -- which is virtually every PAC learnable class -- admits a polynomial time learning algorithm with $\opt + \epsilon$ misclassification error in the presence of RCN.  

Our key contribution is to substantially strengthen Kearns' result by showing that every SQ learnable concept class is also learnable with access to the oracle $\oracle_eta$ in the presence of any noise model with polynomially-bounded magnitude; informally, we show the following: 

\begin{theorem}
    \label{theorem:SQ_learnability}
    Consider any SQ learnable concept class, and any distribution $\mathcal{D}$ corrupted with label-noise of polynomially bounded magnitude. Then, there exists a polynomial time learning algorithm that takes as input samples from $\oracle_eta$, and returns with high probability a hypothesis $h$ such that for any $\epsilon > 0$, $\err_{\mathcal{D}}(h) = \opt + \epsilon$, where $\opt$ is the misclassification error of the optimal classifier.
\end{theorem}

We refer to \Cref{theorem:SQ_learnability} for a more precise statement. In our simulation (\Cref{algorithm:SQ_reduction}), we build on ideas from \cite{DBLP:journals/corr/abs-2006-04787}, and a statistical query decomposition lemma due to \cite{10.1162/153244302760200669}. Despite the many practical applications that motivate having access to the uncertainty of each given label, we consider our contribution to be mostly of theoretical significance. Indeed, one possible interpretation of \Cref{theorem:SQ_learnability} is that the crux of the label noise is not the variance itself, but rather the uncertainty on how the noise is distributed. We remark that even if the learner had access to the noise function, the underlying optimization problem is non-convex; see \cite{DBLP:conf/nips/DiakonikolasGT19}.

\section{Preliminaries}
\label{section:preliminaries}
Throughout this work, we denote with $\mathcal{X} \subseteq \mathbb{R}^d$ the instance space -- or the domain of the samples, while we focus solely on the binary classification problem, i.e. the label space $\mathcal{Y}$ is simply the binary set $\{ \pm 1\}$. A hypothesis is a polynomial time computable function $h : \mathcal{X} \mapsto \{ \pm 1\}$. We will use $h^*$ to represent the Bayesian-optimal classifier, which remains invariant when the noise function satisfies $\eta(\mathbf{x}) < \frac{1}{2}$. The misclassification error of a hypothesis $h$ with respect to distribution $\mathcal{D}$ is defined as $\err_{\mathcal{D}}(h) = \pr_{(\mathbf{x}, y) \sim \mathcal{D}}[h(\mathbf{x}) \neq y]$, while $\err_{\mathcal{D}}(h^*) = \E_{\mathcal{D}_{\mathbf{x}}}[\eta(\mathbf{x})]$. A \emph{weak learning} algorithm produces a hypothesis $h$ such that $\err_{\mathcal{D}}(h) \leq \frac{1}{2} - \frac{1}{p(s)}$, for some fixed polynomial $p(s)$.

Linear threshold functions (LTFs) are Boolean functions $f : \mathbb{R}^d \mapsto \{ \pm 1\}$ of the form $f(\mathbf{x}) = \sign (\langle \mathbf{w}, \mathbf{x} \rangle - \theta)\footnote{recall that function $\sign : \mathbb{R} \mapsto \{ \pm 1 \}$ is defined as $\sign(u) = 1$ for $u \geq 0$ and $\sign(u) = -1$ otherwise.}$, where $\mathbf{w} \in \mathbb{R}^d$ is the weight vector and $\theta \in \mathbb{R}$ is the threshold. We assume -- without any loss of generality -- that the LTF is \emph{homogeneous}, i.e. $\theta = 0$.

If $f$ represents the target function, $\eta$ the noise function, and $\mathcal{D}_{\mathbf{x}}$ the marginal distribution over the instance space, we are using the following notation:

\begin{itemize}
    \item \emph{noiseless} oracle: $\noiseless_oracle(f, \mathcal{D}_{\mathbf{x}})$ returns a labeled example $(\mathbf{x}, y)$, where $\mathbf{x} \sim \mathcal{D}_{\mathbf{x}}$ and $y = f(\mathbf{x})$.
    \item \emph{noisy} oracle: $\noisy_oracle(f, \mathcal{D}_{\mathbf{x}}, \eta)$ returns a noisy labeled example corrupted with noise function $\eta$, as in \Cref{definition:noisy_oracle}. 
    \item \emph{extended noisy} oracle: $\oracle_eta(f, \mathcal{D}_{\mathbf{x}}, \eta)$ returns a noisy labeled example $(\mathbf{x}, y)$ along with the flipping probability at this particular point $\eta(\mathbf{x})$.
\end{itemize}

\subsection{Statistical Query Learning}

Here we provide some basic definitions from the statistical query framework.

\begin{definition}[Statistical Query Model \cite{10.1145/293347.293351}]
    Let $f$ be an unknown target function in a class of Boolean functions $\mathcal{C}$ over $\mathcal{X}$. In the statistical query model the learner interacts with an oracle $\stat(f, \mathcal{D}_{\mathbf{x}})$ that replaces the standard examples oracle $\noiseless_oracle(f, \mathcal{D}_{\mathbf{x}})$. Specifically, $\stat(f, \mathcal{D}_{\mathbf{x}})$ takes as input a statistical query of the form $(\psi, \tau)$, where $\psi: \mathcal{X} \times \{\pm 1 \} \mapsto [-1,1]$ and $\tau \in [0,1]$ the tolerance parameter, and returns any number $v$ such that $|\E_{\mathcal{D}_{\mathbf{x}}} [\psi(\mathbf{x}, f(\mathbf{x}))] - v | \leq \tau$.
\end{definition}

\begin{definition}[Correlation Statistical Queries \cite{10.1162/153244302760200669}]
    A correlational statistical query (CSQ) is a statistical query for a correlation of a function over $\mathcal{X}$ with the target function, namely $\psi(\mathbf{x}, f(\mathbf{x})) = \phi(\mathbf{x}) \cdot f(\mathbf{x})$ for some function $\phi: \mathcal{X} \mapsto [-1,1]$. 
\end{definition}

\begin{definition}[Learning from Statistical Queries]
    A concept class $\mathcal{C}$ over $\mathcal{X}$ is said to be SQ learnable if there exists an algorithm such that for any target function $f \in \mathcal{C}$ and any distribution $\mathcal{D}_{\mathbf{x}}$ in $\mathcal{X}$, it outputs a hypothesis $h$ with $\mathbf{Pr}_{\mathcal{D}_{\mathbf{x}}}[h(\mathbf{x}) \neq f(\mathbf{x})] \leq \epsilon$, for any $\epsilon > 0$, using $\poly(d, 1/\epsilon)$ number of SQ queries and tolerance $1/\poly(d, 1/\epsilon)$. Furthermore, we say that a concept class $\mathcal{C}$ is CSQ learnable if it is SQ learnable with access only to correlational statistical queries. 
\end{definition}

\subsection{Evolvability and CSQ Learnability}

Valiant's model of evolvability is a constrained form of PAC learning, and has been established as a framework for analyzing the computational capabilities of evolutionary processes through sequences of random mutations guided by natural selection \cite{DBLP:journals/jacm/Valiant09}. Providing a formal definition of evolvability would go beyond the scope of our work, and instead, the following very elegant characterization -- implying equivalence between evolvability and CSQ learnability -- will suffice.

\begin{theorem}[\cite{DBLP:conf/stoc/Feldman08}, Theorem 1.1]
    Let $\mathcal{C}$ be a concept class CSQ learnable by an algorithm $\mathcal{A}$ over a class of distributions $\mathcal{D}$. Then, there exists an evolutionary algorithm $N(\mathcal{A})$ such that $\mathcal{C}$ is evolvable by $N(\mathcal{A})$ over $\mathcal{D}$.
\end{theorem}

\begin{theorem}[\cite{DBLP:conf/stoc/Feldman08}, Theorem 4.1]
    If a concept class $\mathcal{C}$ is evolvable over a class of distributions $\mathcal{D}$, then $\mathcal{C}$ is learnable with correlational statistical queries over $\mathcal{D}$.
\end{theorem}

\subsection{Useful Tools}

For some of our proofs, we employ the following standard bound:

\begin{theorem}[Hoeffding's inequality (\cite{10.2307/2282952})]
    \label{theorem:hoeffding}
    Let $Z_1, \dots, Z_n$ be independent random variables with $Z_i \in [a,b]$, for all $i \in [n]$. Then, for all $\epsilon > 0$,
    \begin{equation}
        \pr \left[ \left\lvert \frac{1}{n} \sum_{i=1}^n (Z_i - \E[Z_i]) \right\rvert \geq \epsilon \right] \leq 2 \exp\left( - \frac{2n \epsilon^2}{(b-a)^2} \right).
    \end{equation}
\end{theorem}

\section{Magnitude of the Noise}

This section introduces the \emph{magnitude} of the noise, a parameter that will allow us to analyze both \Cref{algorithm:CSQ_reduction} and \Cref{algorithm:SQ_reduction} for a broad class of noise models with a unifying treatment. We also derive a distribution-independent upper-bound on the magnitude of the Tsybakov noise. 

\begin{definition}
    \label{definition:magnitude}
Let $\mathcal{D}_{\mathbf{x}}$ be a distribution over $\mathcal{X}$. We define the magnitude of a noise function $\eta : \mathcal{X} \mapsto \left[0, \frac{1}{2} \right)$ with respect to distribution $\mathcal{D}_{\mathbf{x}}$ as 

\begin{equation}
    \mathcal{M} = \left( \E_{\mathbf{x} \sim \mathcal{D}_{\mathbf{x}}} \left[ (1 -2\eta(\mathbf{x})) \right] \right)^{-1}.
\end{equation}
\end{definition}

We will always use the magnitude of the noise with respect to the underlying distribution $\mathcal{D}_{\mathbf{x}}$ (and so we may simply say the magnitude of the noise). This parameter reflects how close is the noise on average to the barrier of $\frac{1}{2}$. Nonetheless, we should point out that the difficulty of the instance is not necessarily captured by the magnitude; e.g., a high magnitude RCN instance would be more easily handled than a Massart noise instance with more modest magnitude, mainly due to the symmetry of the former model. 

\paragraph{Super-Polynomial Magnitude} All the guarantees we establish throughout this work are applicable when the magnitude of the noise is polynomially-bounded. In contrast, we point out that a super-polynomial magnitude precludes even the possibility of a weak learner. Indeed, given that $\mathcal{M}^{-1} = \E_{\mathcal{D}_{\mathbf{x}}}[(1 - 2\eta(\mathbf{x}))] = 1 - 2\opt$, it follows that 

\begin{equation}
 \err_{\mathcal{D}}(h^*) = \frac{1}{2} - \frac{1}{2\mathcal{M}}.   
\end{equation}

As a result, if the magnitude of the noise with respect to the underlying distribution is super-polynomial, even the Bayesian-optimal classifier is \emph{not} a weak learner. 

In the following lemma we provide a distribution-free upper-bound on the magnitude of the Tsybakov noise with respect to the parameters of the model. Note that we assume -- without any loss of generality -- that $t_0$ is such that $A t_0^{\frac{\alpha}{1 - \alpha}} \leq 1$.

\begin{lemma}
    \label{lemma:Tsybakov-magnitude}
The magnitude $\mathcal{M}$ of an $(\alpha, A, t_0)$-Tsybakov noise with respect to any distribution $\mathcal{D}_{\mathbf{x}}$ can be upper-bounded as 

\[
    \mathcal{M} \leq \left\{\begin{array}{lr}
        \frac{1}{2\alpha} \left( \frac{A}{1 - \alpha} \right)^{\frac{1 - \alpha}{\alpha}}  & \text{if } t^* \leq t_0, \\ \\
        \left\{ 2t_0\left(1 - At_0^{\frac{\alpha}{1 - \alpha}}\right) \right\}^{-1} & \text{if } t^* > t_0, \\
        \end{array}\right.
  \]
  where

\begin{equation}
    t^* = \left( \frac{1 - \alpha}{A} \right)^{\frac{1 - \alpha}{\alpha}}.
\end{equation}
\end{lemma}

\begin{proof}
Consider some $t \in [0, t_0]$, and let $\mathcal{D}_{\mathbf{x}}$ denote the marginal distribution on the unlabeled points. By definition of the Tsybakov noise condition, the instance space $\mathcal{X}$ may be partitioned into regions $\mathcal{X}_{good}$ and $\mathcal{X}_{bad}$ such that 

\begin{itemize}
    \item $\pr_{\mathbf{x} \sim \mathcal{D}_{\mathbf{x}}}[\mathbf{x} \in \mathcal{X}_{good}] \geq 1 - A t^{\frac{\alpha}{1 - \alpha}}$, and $\eta(\mathbf{x}) \leq \frac{1}{2} - t$ almost surely for all $\mathbf{x} \in \mathcal{X}_{good}$. The points in $\mathcal{X}_{good}$ should be thought of as being corrupted with Massart noise;
    \item $\pr_{\mathbf{x} \sim \mathcal{D}_{\mathbf{x}}}[\mathbf{x} \in \mathcal{X}_{bad}] \leq A t^{\frac{\alpha}{1 - \alpha}}$. The points in $\mathcal{X}_{bad}$ may have flipping probabilities arbitrarily close to $1/2$.
\end{itemize}
As a result, it follows that 

\begin{align}
    \int_{\mathcal{X}} (1 - 2\eta(\mathbf{x})) \mathcal{D}_{\mathbf{x}}(\mathbf{x}) d \mathbf{x} &= \int_{\mathcal{X}_{good}} (1 - 2\eta(\mathbf{x})) \mathcal{D}_{\mathbf{x}}(\mathbf{x}) d \mathbf{x} + \overbrace{\int_{\mathcal{X}_{bad}} (1 - 2\eta(\mathbf{x})) \mathcal{D}_{\mathbf{x}}(\mathbf{x}) d \mathbf{x}}^{> 0} \\
    &> 2t \int_{\mathcal{X}_{good}} \mathcal{D}_{\mathbf{x}}(\mathbf{x}) d \mathbf{x} \\
    &\geq 2t (1 - At^{\frac{\alpha}{1 - \alpha}}),
\end{align}
where in the first line we used that $\eta(\mathbf{x}) < \frac{1}{2}$ for all $\mathbf{x} \in \mathcal{X}_{bad}$ and $\eta(\mathbf{x}) \leq \frac{1}{2} - t$ for all $\mathbf{x} \in \mathcal{X}_{good}$. As a result, we obtain that 

\begin{equation}
    \mathcal{M}^{-1} \geq \sup_{t \in [0, t_0]} \left\{ 2t (1 - At^{\frac{\alpha}{1 - \alpha}}) \right\}.
\end{equation}
Finally, it is easy to verify that 

\[
     \sup_{t \in [0, t_0]} \left\{ 2t (1 - At^{\frac{\alpha}{1 - \alpha}}) \right\} = \left\{\begin{array}{lr}
        2\alpha \left( \frac{1 - \alpha}{A} \right)^{\frac{1 - \alpha}{\alpha}} & \text{if } t^* \leq t_0, \\
        2t_0\left(1 - At_0^{\frac{\alpha}{1 - \alpha}}\right) & \text{if } t^* > t_0. \\
        \end{array}\right.
  \]
  We should mention that when $t^* > t_0$, it follows that  $A t_0^{\frac{\alpha}{1 - \alpha}} \neq 1$. 
\end{proof}

As a special case of this lemma, note that the magnitude of a $\gamma$-Massart noise is upper-bounded by $1/(1 - 2\gamma)$, with the bound being tight for the special case of RCN.

\paragraph{Remark} Throughout this work we endeavor to minimize the misclassification error of a hypothesis $h$ with respect to the \emph{noisy} distribution $\mathcal{D}$, i.e. attain $\err_{\mathcal{D}}(h) = \opt + \epsilon$, for any $\epsilon$. However, one could ask how would such a guarantee translate in the underlying noiseless or realizable instance; in other words, the question is whether having a hypothesis $h$ such that $\err_{\mathcal{D}}(h) \leq \opt + \epsilon$ implies that $\pr_{\mathcal{D}_{\mathbf{x}}}[h(\mathbf{x}) \neq f(\mathbf{x})] \leq \epsilon' $, for some $\epsilon'$ that depends polynomially on $\epsilon$. In the Massart as well as the Tsybakov model this is indeed the case, although it does not hold for some noise functions with polynomially bounded magnitude; we refer the reader to \Cref{appendix:realizable} for additional discussion.

\section{CSQ Learnability Implies Noise Tolerance}

In this section we analyze an algorithm devised by \cite{DBLP:journals/corr/abs-2006-04787}. Specifically, we extend their analysis in the presence of any noise model with polynomially-bounded magnitude, while they only provided an analysis for the Massart noise model. Their main insight was to consider an artificial noiseless classification problem on a distribution $\mathcal{D}_{\mathbf{x}}'$ transformed according to the noise function. More precisely, 

\begin{equation}
    \label{equaiton:transformed_dist}
    \mathcal{D}_{\mathbf{x}}'(\mathbf{x}) = \frac{1}{Z} \mathcal{D}_{\mathbf{x}}(\mathbf{x})(1 - 2\eta(\mathbf{x})),
\end{equation}
where $Z$ here serves as a normalization constant; notice that $Z = \mathcal{M}^{-1}$. Interestingly, this artificial classification problem transfers the noise from the label space to the instance space. In this way, it is connected with several other noise models in which the adversary perturbs the distribution over the instance space; e.g., see \cite{DBLP:conf/alt/BshoutyEK98}. The first observation is that solving the artificial noiseless problem suffices. 

\begin{lemma}
    \label{lemma:opt_reduction}
    Consider some target function $f: \mathcal{X} \mapsto \{\pm 1\}$, a distribution $\mathcal{D}_{\mathbf{x}}$ over $\mathcal{X}$, and $\mathcal{D}_{\mathbf{x}}'$ as defined in \eqref{equaiton:transformed_dist}. If $h$ is a hypothesis such that $\mathbf{Pr}_{\mathbf{x} \sim \mathcal{D}_{\mathbf{x}}'}[h(\mathbf{x}) \neq f(\mathbf{x})] \leq \epsilon$, it follows that $\mathbf{Pr}_{(\mathbf{x}, y) \sim \mathcal{D}}[h(\mathbf{x}) \neq y] \leq \opt + \epsilon$.
\end{lemma}

\begin{proof}
First of all, we have that 

\begin{align}
    \pr_{\mathbf{x} \sim \mathcal{D}_{\mathbf{x}}'}[h(\mathbf{x}) \neq f(\mathbf{x})] &= \frac{1}{Z} \E_{\mathbf{x} \sim \mathcal{D}_{\mathbf{x}}} [(1 - 2\eta(\mathbf{x})) \mathds{1}\{ h(\mathbf{x}) \neq f(\mathbf{x})\}] \\
    &\geq \E_{\mathbf{x} \sim \mathcal{D}_{\mathbf{x}}} [(1 - 2\eta(\mathbf{x})) \mathds{1}\{ h(\mathbf{x}) \neq f(\mathbf{x})\}],
\end{align}
where the last inequality follows from $Z \leq 1$. Moreover, we obtain that

\begin{align}
    \mathbf{Pr}_{(\mathbf{x}, y) \sim \mathcal{D}} [h(\mathbf{x}) \neq y] &= 
    \E_{\mathbf{x} \sim \mathcal{D}_{\mathbf{x}}} [(1 - \eta(\mathbf{x})) \mathds{1} \{ h(\mathbf{x}) \neq f(\mathbf{x}) \} ] + \E_{\mathbf{x} \sim \mathcal{D}_{\mathbf{x}}} [\eta(\mathbf{x}) \mathds{1} \{ h(\mathbf{x}) = f(\mathbf{x}) \} ] \\
    &= \E_{\mathbf{x} \sim \mathcal{D}_{\mathbf{x}}}[\eta(\mathbf{x})] +  \E_{\mathbf{x} \sim \mathcal{D}_{\mathbf{x}}} [(1 - 2\eta(\mathbf{x})) \mathds{1} \{ h(\mathbf{x}) \neq f(\mathbf{x}) \} ] \\
    &\leq \opt + \epsilon,
\end{align}
where we used that $\opt = \E_{\mathbf{x} \sim \mathcal{D}_{\mathbf{x}}}[\eta(\mathbf{x})]$ and $\mathbf{Pr}_{\mathbf{x} \sim \mathcal{D}_{\mathbf{x}}'}[h(\mathbf{x} \neq f(\mathbf{x}))] \leq \epsilon$.
\end{proof}

Importantly, the next lemma implies that although we do not have access to distribution $\mathcal{D}_{\mathbf{x}}'$, we could simulate correlational statistical queries on $\mathcal{D}_{\mathbf{x}}'$ through the empirically observed distribution $\mathcal{D}$, if we knew the value of the normalization constant $Z$.

\begin{lemma}[\cite{DBLP:journals/corr/abs-2006-04787}, Fact 7.5]
    \label{lemma:CSQ_simulation}
    Consider some target function $f : \mathcal{X} \mapsto \{\pm 1\}$, a distribution $\mathcal{D}_{\mathbf{x}}$ over $\mathcal{X}$, and $\mathcal{D}_{\mathbf{x}}'$ as defined in \eqref{equaiton:transformed_dist}. Then, for any correlational statistical query $\psi(\mathbf{x}, f(\mathbf{x})) = \phi(\mathbf{x}) \cdot f(\mathbf{x})$,

\begin{equation}
    \E_{\mathbf{x} \sim \mathcal{D}_{\mathbf{x}}'}[\psi(\mathbf{x}, f(\mathbf{x}))] = \frac{1}{Z} \cdot \E_{(\mathbf{x}, y) \sim \mathcal{D}}[\psi(\mathbf{x}, y)].
\end{equation}
\end{lemma}

\begin{proof}
It follows that 

\begin{align}
    \E_{(\mathbf{x}, y) \sim \mathcal{D}}[\psi(\mathbf{x}, y)] &= \int_{\mathcal{X}} \phi(\mathbf{x}) f(\mathbf{x}) (1-2\eta(\mathbf{x})) \mathcal{D}_{\mathbf{x}}(\mathbf{x}) d \mathbf{x} \\
    &= Z \int_{\mathcal{X}} \phi(\mathbf{x}) f(\mathbf{x}) \mathcal{D}_{\mathbf{x}}'(\mathbf{x}) d \mathbf{x} \\
    &= Z \E_{\mathbf{x} \sim \mathcal{D}_{\mathbf{x}}'}[\phi(\mathbf{x}) f(\mathbf{x})] \\
    &= Z \E_{(\mathbf{x}, y) \sim \mathcal{D}'} [\psi(\mathbf{x}, y)].
\end{align}
\end{proof}

The main idea in the algorithm of \cite{DBLP:journals/corr/abs-2006-04787} is to search in a brute-force manner for the correct normalization constant in order to simulate the correlational statistical queries on distribution $\mathcal{D}_{\mathbf{x}}'$.

\begin{algorithm}[!h]
\textbf{Input}:
\begin{itemize}
    \item[(i)] Algorithm $\mathcal{A}$ which efficiently and distribution-independently learns a target $f \in \mathcal{C}$ to error $\epsilon$ with at most $q$ CSQs of tolerance $\tau \in [0,1]$
    \item[(ii)] Sampling access to distribution $\mathcal{D}$ corrupted with \emph{unknown} noise $\eta(\mathbf{x})$ of magnitude $\mathcal{M} \leq C$
    \item[(iii)] Accuracy parameter $\epsilon > 0$ 
    \item[(iv)] Confidence parameter $\delta > 0$
\end{itemize}
\textbf{Output}: Hypothesis $h$ such that $\pr_{\mathcal{D}}[h(\mathbf{x}) \neq y] \leq \opt + \epsilon$\\
Set $\tau' := \tau/(2 C^2), i:= 0$ \\
\For{$\tilde{Z} \in [0, \tau', 2\tau', \dots, 1]$}{
    
    \begin{itemize}
        \item Simulate $\mathcal{A}$ on $\mathcal{D}_{\mathbf{x}}'$: answer every correlational statistical query $\psi(\mathbf{x}, f(\mathbf{x}))$ on $\mathcal{D}_{\mathbf{x}}'$ with $\widehat{\E}_{\mathcal{D}}[\psi(\mathbf{x}, y)]/\tilde{Z}$, where $\widehat{\E}_{\mathcal{D}}[\psi(\mathbf{x}, y)]$ is the empirical estimate of $ \E_{(\mathbf{x}, y) \sim \mathcal{D}}[\psi(\mathbf{x}, y)]$ formed from $\mathcal{O} \left( C^2\log(q/\delta)/\tau^2 \right)$ samples
        \item Let $h_i$ be the output of $\mathcal{A}$
        \item Estimate $\err_{\mathcal{D}}(h_i) = \mathbf{Pr}_{(\mathbf{x}, y) \sim \mathcal{D}} [h_i(\mathbf{x}) \neq y]$ from $\mathcal{O}\left(\log(C^2/(\delta \tau))/\epsilon^2 \right)$ samples
        \item i := i + 1
    \end{itemize}
  }
\textbf{return} the hypothesis $h_i$ that incurs the smallest empirical error.
 \caption{$\algocsq$}
 \label{algorithm:CSQ_reduction}
\end{algorithm}

\begin{theorem}
    \label{theorem:CSQ_learnability}
    Consider a concept class of boolean functions $\mathcal{C}$ in $\mathcal{X} \subseteq \mathbb{R}^d$ which is CSQ learnable by an algorithm $\mathcal{A}$. Then, for any distribution $\mathcal{D}$ corrupted with noise of magnitude upper-bounded by $C$, $\algocsq$ takes as input a $\poly(d, 1/\epsilon, 1/\delta, C)$ number of samples, runs in $\poly(d, 1/\epsilon, 1/\delta, C)$ time, and returns a hypothesis $h$ such that $\pr_{(\mathbf{x}, y) \sim \mathcal{D}}[h(\mathbf{x}) \neq y] \leq \opt + \epsilon$ with probability at least $1 - \delta$, for any $\epsilon > 0$ and $\delta > 0$, where $\opt = \err_{\mathcal{D}}(h^*)$.
\end{theorem}

\begin{proof}
First of all, given that $\mathcal{A}$ efficiently learns the concept class $\mathcal{C}$ up to an $\epsilon$ error, it follows that $q = \poly(d, 1/\epsilon)$ and $1/\tau = \poly(d, 1/\epsilon)$. For some iteration in the main loop of the algorithm, $\tilde{Z}$ will be such that $Z \leq \tilde{Z} \leq Z + \tau'$. For this particular $\tilde{Z}$, it follows that $| 1/Z - 1/\tilde{Z}| \leq \tau'/Z^2 \leq \tau' C^2 = \tau/2$, where we used that $Z \geq 1/C$.

Now consider any correlational statistical query $\psi(\mathbf{x}, y)$; we have to establish that when our guess for parameter $Z$ is close to the actual value, every query of algorithm $\mathcal{A}$ is simulated correctly with high probability. Indeed, \Cref{lemma:CSQ_simulation} implies that 

\begin{equation}
    \label{equation:sim_csq}
    \left\lvert \frac{1}{\tilde{Z}} \E_{(\mathbf{x}, y) \sim \mathcal{D}}[\psi(\mathbf{x}, y)] - \E_{\mathbf{x} \sim \mathcal{D}_{\mathbf{x}}'} [\psi(\mathbf{x}, f(\mathbf{x}))] \right\rvert = \left\lvert \frac{1}{\tilde{Z}} - \frac{1}{Z} \right\rvert \left\lvert \E_{(\mathbf{x}, y) \sim \mathcal{D}} [\psi(\mathbf{x}, y)] \right\rvert \leq \tau/2,
\end{equation}
where $\mathcal{D}'$ is defined as in \Cref{lemma:CSQ_simulation}. Moreover, let $\widehat{\E}_{\mathcal{D}}[\psi(\mathbf{x}, y)]$ be the empirical estimate of $\E_{(\mathbf{x}, y) \sim \mathcal{D}}[\psi(\mathbf{x}, y)]$ formed from $\mathcal{O} \left( C^2\log(q/\delta)/\tau^2 \right)$ samples. Given that $|\psi(\mathbf{x}, y)| \leq 1$, Hoeffding's inequality implies that with probability at least $1 - \delta/q$,

\begin{equation}
    \label{equation:emp_sim}
    \left\vert \E_{(\mathbf{x}, y) \sim \mathcal{D}} [\psi(\mathbf{x}, y)] - \widehat{\E}_{\mathcal{D}}[\psi(\mathbf{x}, y)] \right\rvert  \leq \frac{\tau}{2C} \leq \frac{\tau \tilde{Z}}{2}.
\end{equation}
As a result, combining \eqref{equation:sim_csq} and \eqref{equation:emp_sim} yields that with probability at least $1 - \delta/q$, 

\begin{equation}
  \left\lvert \frac{1}{\tilde{Z}} \widehat{\E}_{\mathcal{D}}[\psi(\mathbf{x}, y)] - \E_{\mathbf{x} \sim \mathcal{D}_{\mathbf{x}}'} [\psi(\mathbf{x}, f(\mathbf{x}))] \right\rvert \leq \frac{\tau'}{Z^2} \leq \tau.
\end{equation}

By the union bound, we obtain that for the $\tilde{Z}$ that satisfies $Z \leq \tilde{Z} \leq Z + \tau'$, all of the $q$ CSQ queries made by algorithm $\mathcal{A}$ are answered correctly up to error $\tau$ with probability at least $1 - \delta$. Then, for this particular iteration the output hypothesis $h$ of algorithm $\mathcal{A}$ satisfies $\pr_{\mathbf{x} \sim \mathcal{D}_{\mathbf{x}}'}[h(\mathbf{x}) \neq f(\mathbf{x})] \leq \epsilon$, which -- by \Cref{lemma:opt_reduction} -- implies that $\pr_{(\mathbf{x}, y) \sim \mathcal{D}}[h(\mathbf{x}) \neq y] \leq \opt + \epsilon$. Finally, let $\widehat{\pr}_{\mathcal{D}}[h(\mathbf{x}) \neq y]$ be the empirical estimate of $\pr_{(\mathbf{x}, y) \sim \mathcal{D}}[h(\mathbf{x}) \neq y]$. If we invoke $\mathcal{O}\left(\log(1/\delta)/\epsilon^2 \right)$ samples, we obtain that with probability at least $1 - \delta$, 

\begin{equation}
    \left\lvert \widehat{\pr}_{\mathcal{D}}[h(\mathbf{x}) \neq y] - \pr_{(\mathbf{x}, y) \sim \mathcal{D}}[h(\mathbf{x}) \neq y] \right\rvert \leq \epsilon.
\end{equation}

Thus, by the union bound $\mathcal{O}\left( \log(N/\delta)/\epsilon^2 \right)$ samples suffice to guarantee that the estimation error is up to $\epsilon$ in every iteration with probability at least $1 - \delta$, where $N = \mathcal{O}(C^2/\tau)$ is the number of iterations of the main loop in the algorithm. Consequently, the output of the algorithm $h$ satisfies, with probability at least $1 - 2\delta$, $\pr_{(\mathbf{x}, y) \sim \mathcal{D}}[h(\mathbf{x}) \neq y] \leq \opt + 3\epsilon$. Finally, rescaling $\epsilon$ and $\delta$ concludes the proof.
\end{proof}

Connections of the type established in \Cref{theorem:CSQ_learnability} are quite compelling, given that every evolutionary algorithm formulated in the literature will automatically imply noise tolerance under very challenging noise models. Unfortunately, distribution-independent evolvability is a rather rare occurrence, and to the best of our knowledge the only known result was obtained by Feldman (\cite{DBLP:conf/colt/Feldman09}, Theorem 18) for the trivial class of single points. For this reason, we will employ $\algocsq$ to simulate correlational statistical queries on $\mathcal{D}_{\mathbf{x}}'$ in the distribution-specific setting. In particular, \Cref{theorem:CSQ_learnability} requires that the concept class is evolvable over $\mathcal{D}_{\mathbf{x}}'$ for any noise function. Thus, the following theorem follows directly from Theorem 15 of \cite{DBLP:conf/colt/KanadeVV10}, which implies evolvability of LTFs over any spherically symmetric distribution.  

\begin{theorem}
    \label{theorem:Tsybakov}
    Let $f$ be an unknown linear threshold function and $\mathcal{D}$ a distribution on $\mathcal{X} \times \{ \pm 1\}$ corrupted with spherically symmetric $(\alpha, A, t_0)$-Tsybakov noise of magnitude upper-bounded by $C = C(\alpha, A, t_0)$, such that the marginal distribution $\mathcal{D}_{\mathbf{x}}$ over $\mathcal{X}$ is spherically symmetric. Then, there exists an algorithm that takes as input a $\poly(d, 1/\epsilon, 1/\delta, C)$ number of samples, runs in $\poly(d, 1/\epsilon, 1/\delta, C)$ time, and returns a hypothesis $h$ such that $\pr_{(\mathbf{x}, y) \sim \mathcal{D}}[h(\mathbf{x}) \neq y] \leq \opt + \epsilon$ with probability at least $1 - \delta$, for any $\epsilon > 0$ and $\delta > 0$, where $\opt = \err_{\mathcal{D}}(h^*)$.
\end{theorem}

\section{Robust Learning with Extended Noisy Oracle}

In this section we are considering a more benign adversary. Specifically, we posit access to an extended noisy oracle $\oracle_eta(f, \mathcal{D}_{\mathbf{x}}, \eta)$, such that every time we invoke $\oracle_eta(f, \mathcal{D}_{\mathbf{x}}, \eta)$ it returns a noisy labeled example along with the corresponding flipping probability. We stress that this particular model is a strong extension of the RCN model with known noise rate, which is a non-trivial and widely studied model in the literature of machine learning. In this context, our main contribution is to extend the following celebrated lemma:

\begin{lemma}[\cite{10.1145/293347.293351}]
    \label{lemma:Kearns-simulation}
    For any query function $\psi$ and target function $f$, $\E_{\mathcal{D}_{\mathbf{x}}}[\psi(\mathbf{x}, f(\mathbf{x}))]$ can, with probability at least $1 - \delta$, be estimated within $\tau$ using $\mathcal{O}(\log(1/\delta)/((1-2\eta)^2 \tau^2))$ samples from the noisy-RCN oracle.
\end{lemma}

Our main idea is to decompose a general statistical query into a correlational and a \emph{target independent} statistical query; more precisely, we say that a statistical query is independent of the target if the query function $\psi(\mathbf{x}, f(\mathbf{x}))$ is a function of $\mathbf{x}$ alone, i.e. it does not depend on the value of the second parameter. Then, our main technical ingredients (\Cref{lemma:TI-sim} and \Cref{lemma:CSQ-sim}) establish that every component can be efficiently simulated with high probability with access to $\oracle_eta$. 

\begin{lemma}[\cite{10.1162/153244302760200669}, Lemma 30\footnote{It should be noted that this decomposition was first implicitly employed in \cite{DBLP:conf/stoc/BlumFJKMR94}.}]
    \label{lemma:decomposition}    
    Any statistical query $(\psi, \tau)$ with respect to any distribution $\mathcal{D}_{\mathbf{x}}'$ can be answered by adding the value of a target independent statistical query $(\phi_{TI}, \tau/2)$ to the value of a correlational statistical query $(\phi_{CSQ}, \tau/2)$; specifically, 
    
    \begin{itemize}
        \item $\phi_{TI}(\mathbf{x}) = (\psi(\mathbf{x}, 1) + \psi(\mathbf{x}, -1))/2$;
        \item $\phi_{CSQ}(\mathbf{x}) = (\psi(\mathbf{x}, 1) - \psi(\mathbf{x}, -1))/2$.
    \end{itemize}
\end{lemma}

\begin{proof}
    If $f$ represents the target function, the claim follows from the following observation:
    \begin{align}
        \E_{\mathbf{x} \sim \mathcal{D}_{\mathbf{x}}'}[\psi(\mathbf{x}, f(\mathbf{x}))] &= \E_{\mathbf{x} \sim \mathcal{D}_{\mathbf{x}}'} \left[ \psi(\mathbf{x}, -1) \frac{1 - f(\mathbf{x})}{2} + \psi(\mathbf{x}, 1) \frac{1 +f(\mathbf{x})}{2} \right] \\
    &= \E_{\mathbf{x} \sim \mathcal{D}_{\mathbf{x}}'} \left[ \frac{\psi(\mathbf{x}, 1) - \psi(\mathbf{x}, -1)}{2} f(\mathbf{x}) \right] + \E_{\mathbf{x} \sim \mathcal{D}_{\mathbf{x}}'} \left[ \frac{\psi(\mathbf{x},1) + \psi(\mathbf{x}, -1)}{2} \right].
    \end{align}
\end{proof}

\begin{lemma}
    \label{lemma:TI-sim}
    Consider a distribution $\mathcal{D}_{\mathbf{x}}$ on $\mathcal{X}$ and
    $\mathcal{D}_{\mathbf{x}}'$ with density $\mathcal{D}_{\mathbf{x}}'(\mathbf{x}) = \mathcal{M} (1-2\eta(\mathbf{x})) \mathcal{D}_{\mathbf{x}}(\mathbf{x})$, where $\mathcal{M} \leq C$. If we have access to the oracle $\oracle_eta(f, \mathcal{D}_{\mathbf{x}}, \eta)$ we can approximate any target independent statistical query of tolerance $\tau$ on distribution $\mathcal{D}_{\mathbf{x}}'$ with $\mathcal{O}\left( C^2 \log(1/\delta)/\tau^2 \right)$ samples from $\oracle_eta(f, \mathcal{D}_{\mathbf{x}}, \eta)$ with probability at least $1 - \delta$.
\end{lemma}

\begin{proof}
    Let $(\phi', \tau)$ represent the target independent statistical query. In the interest of simplifying our argument we notice that 
    
    \begin{equation}
        \E_{\mathbf{x} \sim \mathcal{D}_{\mathbf{x}}'}[\phi'(\mathbf{x})] = \E_{\mathbf{x} \sim \mathcal{D}_{\mathbf{x}}'} \left[ -1 + 2 \frac{1 + \phi'(\mathbf{x})}{2} \right] = -1 + 2 \E_{\mathbf{x} \sim \mathcal{D}_{\mathbf{x}}'}[\phi(\mathbf{x})],
    \end{equation}
    where $\phi(\mathbf{x}) = (1 + \phi'(\mathbf{x}))/2$. Thus, it suffices to simulate the statistical query $(\phi, \tau/2)$ on $\mathcal{D}_{\mathbf{x}}'$, where $\phi$ takes values in $[0,1]$. If $Z = \mathcal{M}^{-1} = \E_{\mathcal{D}_{\mathbf{x}}}[1 - 2\eta(\mathbf{x})]$, we have that 
    
    \begin{equation}
        \E_{\mathbf{x} \sim \mathcal{D}_{\mathbf{x}}'}[\phi(\mathbf{x})] = \frac{1}{Z} \int_{\mathcal{X}} \phi(\mathbf{x}) (1 - 2\eta(\mathbf{x})) \mathcal{D}_{\mathbf{x}}(\mathbf{x}) d \mathbf{x} = \frac{1}{Z} \E_{\mathbf{x} \sim \mathcal{D}_{\mathbf{x}}}[\phi(\mathbf{x}) (1 - 2\eta(\mathbf{x}))].
    \end{equation}
    Let $\widehat{Z}$ be the empirical estimate of $\E_{\mathcal{D}_{\mathbf{x}}}[1 - 2\eta(\mathbf{x})]$ formed from $\mathcal{O}\left( \log(1/\delta)/(\tau')^2 \right)$ samples of $\oracle_eta(f, \mathcal{D}_{\mathbf{x}}, \eta)$, for some $\delta > 0$ and $\tau' := \tau/(2C)$. Given that $0 \leq 1 - 2\eta(\mathbf{x}) \leq 1, \forall \mathbf{x} \in \mathcal{X}$, Hoeffding's inequality implies that $|\widehat{Z} - Z| < \tau'/2$, with probability at least $1 - \delta$. Thus, if we let $\widehat{Z} := \widehat{Z} + \tau'/2$, we obtain that $Z < \widehat{Z} < Z + \tau'$, with probability at least $1 - \delta$. Furthermore, let $\widehat{\E}_{\mathcal{D}_{\mathbf{x}}}[\phi(\mathbf{x})(1 - 2\eta(\mathbf{x}))]$ be the empirical estimate of $\E_{\mathcal{D}_{\mathbf{x}}}[\phi(\mathbf{x})(1 - 2\eta(\mathbf{x}))]$ formed from $\mathcal{O}(\log(1/\delta)/(\tau')^2)$ of $\oracle_eta(f, \mathcal{D}_{\mathbf{x}}, \eta)$. If we increment the estimate by $\tau'/2$ we can again guarantee that $\E_{\mathcal{D}_{\mathbf{x}}}[\phi(\mathbf{x})(1 - 2\eta(\mathbf{x}))] < \widehat{\E}_{\mathcal{D}_{\mathbf{x}}}[\phi(\mathbf{x})(1 - 2\eta(\mathbf{x}))] < \E_{\mathcal{D}_{\mathbf{x}}}[\phi(\mathbf{x})(1 - 2\eta(\mathbf{x}))] + \tau'$, with probability at least $1 - \delta$. Indeed, given that $0 \leq \phi(\mathbf{x})(1-2\eta(\mathbf{x})) \leq 1, \forall \mathbf{x} \in \mathcal{X}$, we can directly apply Hoeffding's inequality. As a result, with probability at least $1 - 2\delta$ we have that
    
    \begin{equation}
        \label{equation:EZ-upper_bound}
        \frac{\widehat{\E}_{\mathcal{D}_{\mathbf{x}}}[\phi(\mathbf{x})(1 - 2\eta(\mathbf{x}))]}{\widehat{Z}} < \frac{\E_{\mathcal{D}_{\mathbf{x}}}[\phi(\mathbf{x})(1 - 2\eta(\mathbf{x}))] + \tau'}{Z} \leq \E_{\mathcal{D}_{\mathbf{x}}'}[\phi(\mathbf{x})] + \tau' C = \E_{\mathcal{D}_{\mathbf{x}}'}[\phi(\mathbf{x})] + \frac{\tau}{2},
    \end{equation}
    
    \begin{equation}
        \label{equation:EZ-interm}
        \frac{\widehat{\E}_{\mathcal{D}_{\mathbf{x}}}[\phi(\mathbf{x})(1 - 2\eta(\mathbf{x}))]}{\widehat{Z}} > \frac{\E_{\mathcal{D}_{\mathbf{x}}}[\phi(\mathbf{x})(1 - 2\eta(\mathbf{x}))]}{Z + \tau'} \geq \frac{1}{1 + \tau/2} \E_{\mathcal{D}_{\mathbf{x}}'}[\phi(\mathbf{x})],
    \end{equation}
    where in the final bound we used that $\tau' \leq \tau Z/2$. Thus, it follows from \eqref{equation:EZ-interm} that 
    
    \begin{equation}
        \label{equation:EZ-lower_bound}
        \frac{\widehat{\E}_{\mathcal{D}_{\mathbf{x}}}[\phi(\mathbf{x})(1 - 2\eta(\mathbf{x}))]}{\widehat{Z}} - \E_{\mathcal{D}_{\mathbf{x}}'}[\phi(\mathbf{x})] > \E_{\mathcal{D}_{\mathbf{x}}'}[\phi(\mathbf{x})] \left( \frac{1}{1 + \tau/2} - 1 \right) \geq - \E_{\mathcal{D}_{\mathbf{x}}'}[\phi(\mathbf{x})] \frac{\tau}{2} \geq - \frac{\tau}{2},
    \end{equation}
    since $\tau > 0$ and $0 \leq \E_{\mathcal{D}_{\mathbf{x}}'}[\phi(\mathbf{x})] \leq 1$. As a result, if we combine \eqref{equation:EZ-upper_bound} and \eqref{equation:EZ-lower_bound} we obtain that 
    
    \begin{equation}
        - \frac{\tau}{2} < \frac{\widehat{\E}_{\mathcal{D}_{\mathbf{x}}}[\phi(\mathbf{x})(1 - 2\eta(\mathbf{x}))]}{\widehat{Z}} - \E_{\mathcal{D}_{\mathbf{x}}'}[\phi(\mathbf{x})] < \frac{\tau}{2},
    \end{equation}
    with probability at least $1 - 2\delta$; finally, rescaling $\delta := \delta/2$ concludes the proof.
\end{proof}

\begin{lemma}
    \label{lemma:CSQ-sim}
    Consider a distribution $\mathcal{D}_{\mathbf{x}}$ on $\mathcal{X}$ and $\mathcal{D}_{\mathbf{x}}'$ with density $\mathcal{D}_{\mathbf{x}}'(\mathbf{x}) = \mathcal{M} (1-2\eta(\mathbf{x})) \mathcal{D}_{\mathbf{x}}(\mathbf{x})$, where $\mathcal{M} \leq C$. If we have access to the oracle $\oracle_eta(f, \mathcal{D}_{\mathbf{x}}, \eta)$ we can approximate any correlational statistical query of tolerance $\tau$ on distribution $\mathcal{D}_{\mathbf{x}}'$ with $\mathcal{O}\left( C^4 \log(1/\delta)/\tau^2 \right)$ samples from $\oracle_eta(f, \mathcal{D}_{\mathbf{x}}, \eta)$ with probability at least $1 - \delta$.
\end{lemma}

\begin{proof}
    Let $Z = \mathcal{M}^{-1}$ and $\psi(\mathbf{x}, f(\mathbf{x})) = \phi(\mathbf{x}) f(\mathbf{x})$ the input query. Every correlational statistical query on distribution $\mathcal{D}_{\mathbf{x}}'$ can be expressed as 
    
    \begin{equation}
        \E_{\mathbf{x} \sim \mathcal{D}_{\mathbf{x}}'}[\phi(\mathbf{x}) f(\mathbf{x})] = \frac{1}{Z} \int_{\mathcal{X}} \phi(\mathbf{x}) f(\mathbf{x}) (1 - 2\eta(\mathbf{x})) \mathcal{D}_{\mathbf{x}}(\mathbf{x}) d\mathbf{x} = \frac{1}{Z} \E_{(\mathbf{x}, y) \sim \mathcal{D}}[\phi(\mathbf{x}) y].
    \end{equation}
    
    Let $\widehat{Z}$ be the empirical estimate of $Z$ from $\mathcal{O}(\log(1/\delta)/(\tau')^2)$ samples of $\oracle_eta(f, \mathcal{D}_{\mathbf{x}}, \eta)$, for some $\delta > 0$ and $\tau' := \tau/(2C^2)$. If we increment our estimate by $\tau'/2$, it follows that $Z < \widehat{Z} < Z + \tau'$ with probability at least $1 - \delta$. Thus, we obtain that 
    
    \begin{equation}
        \label{equation:EZ-1}
        \left\lvert \frac{1}{Z} \E_{(\mathbf{x}, y) \sim \mathcal{D}}[\phi(\mathbf{x}) y] - \frac{1}{\widehat{Z}} \E_{(\mathbf{x}, y) \sim \mathcal{D}}[\phi(\mathbf{x}) y] \right\rvert \leq \frac{\tau'}{Z^2} \leq \tau' C^2 = \frac{\tau}{2}.
    \end{equation}
    
    Moreover, let $\widehat{E}_{\mathcal{D}}[\phi(\mathbf{x}) y]$ the empirical estimate of $\E_{\mathcal{D}}[\phi(\mathbf{x}) y]$. For $Z < \widehat{Z}$, Hoeffding's inequality implies that $\mathcal{O}(C^2 \log(1/\delta)/\tau^2)$ samples suffice so that 
    
    \begin{equation}
        \label{equation:EZ-2}
        \left\lvert \frac{1}{\widehat{Z}} \widehat{\E}_{\mathcal{D}}[\phi(\mathbf{x}) y] - \frac{1}{\widehat{Z}} \E_{(\mathbf{x}, y) \sim \mathcal{D}}[\phi(\mathbf{x}) y] \right\rvert < \frac{\tau}{2 \widehat{Z} C} < \frac{\tau}{2 Z C} < \frac{\tau}{2},
    \end{equation}
    with probability at least $1 - \delta$. Thus, combining \eqref{equation:EZ-1} and \eqref{equation:EZ-2} we obtain that with probability at least $1 - 2\delta$,
    
    \begin{equation}
        \left\lvert \frac{1}{\widehat{Z}} \widehat{E}_{\mathcal{D}}[\phi(\mathbf{x}) y] - \E_{\mathcal{D}_{\mathbf{x}}'} [\phi(\mathbf{x}) f(\mathbf{x})] \right\rvert < \tau.
    \end{equation}
\end{proof}

Next, we are ready to establish the main theorem of this section. 

\begin{algorithm}[!h]
\textbf{Input}:
\begin{itemize}
    \item[(i)] Algorithm $\mathcal{A}$ which efficiently and distribution-independently learns a target $f \in \mathcal{C}$ to error $\epsilon$ with at most $q$ SQs of tolerance $\tau \in [0,1]$
    \item[(ii)] Sampling access to the extended noisy oracle $\oracle_eta(f, \mathcal{D}_{\mathbf{x}}, \eta)$, where the noise $\eta(\mathbf{x})$ has bounded magnitude $\mathcal{M} \leq C$
    \item[(iii)] Accuracy parameter $\epsilon > 0$ 
    \item[(iv)] Confidence parameter $\delta > 0$
\end{itemize}
\textbf{Output}: Hypothesis $h$ such that $\pr_{\mathcal{D}}[h(\mathbf{x}) \neq y] \leq \opt + \epsilon$ \\
Simulate algorithm $\mathcal{A}$ on distribution $\mathcal{D}_{\mathbf{x}}'$:

\begin{enumerate}
    \item For every statistical query $(\psi, \tau)$ that $\mathcal{A}$ makes, set the following:
    
    \begin{itemize}
        \item $\phi_{TI}(\mathbf{x}) = (\psi(\mathbf{x}, 1) + \psi(\mathbf{x}, -1))/2$
        \item $\phi_{CSQ}(\mathbf{x}) = (\psi(\mathbf{x}, 1) - \psi(\mathbf{x}, -1))/2$
    \end{itemize}
    
    
    \item Let $v_{TI}$ be the estimate of $(\phi_{TI}, \tau/2)$ on $\mathcal{D}_{\mathbf{x}}'$, according to \Cref{lemma:TI-sim}
    
    \item Let $v_{CSQ}$ be the estimate of $(\phi_{CSQ}, \tau/2)$ on $\mathcal{D}_{\mathbf{x}}'$, according to \Cref{lemma:CSQ-sim}
    
    
    \item Answer the query $(\psi, \tau)$ with value $v = v_{CSQ} + v_{TI}$
    
\end{enumerate}

\textbf{return} the output of algorithm $\mathcal{A}$
 \caption{$\algosq$}
 \label{algorithm:SQ_reduction}
\end{algorithm}

\begin{theorem}
    \label{theorem:SQ_learnability}
    Consider a concept class of boolean functions $\mathcal{C}$ in $\mathcal{X} \subseteq \mathbb{R}^d$ which is SQ learnable by an algorithm $\mathcal{A}$. Then, if we have access to the extended noisy oracle $\oracle_eta(f, \mathcal{D}_{\mathbf{x}}, \eta)$ and the magnitude of the noise is at most $C$, $\algosq$ takes as input a $\poly(d, 1/\epsilon, 1/\delta, C)$ number of samples, runs in $\poly(d, 1/\epsilon, 1/\delta, C)$ time, and returns a hypothesis $h$ such that $\pr_{(\mathbf{x}, y) \sim \mathcal{D}}[h(\mathbf{x}) \neq y] \leq \opt + \epsilon$ with probability at least $1 - \delta$, for any $\epsilon > 0$ and $\delta > 0$, where $\opt = \err_{\mathcal{D}}(h^*)$.
\end{theorem}

\begin{proof}
    First of all, given that $\mathcal{A}$ efficiently learns up to an $\epsilon$ error the concept class $\mathcal{C}$, it follows that $q = \poly(d, 1/\epsilon)$ and $1/\tau = \poly(d, 1/\epsilon)$. \Cref{lemma:CSQ-sim} implies that with probability at least $1 - \delta/q$, $| \E_{\mathcal{D}_{\mathbf{x}}'}[\phi_{CSQ}(\mathbf{x}) f(\mathbf{x})] - v_{CSQ} | \leq \tau/2$. Likewise, \Cref{lemma:TI-sim} implies that with probability at least $1 - \delta/q$, $|\E_{\mathcal{D}_{\mathbf{x}}'}[\phi_{TI}(\mathbf{x})] - v_{TI} | \leq \tau/2$. Thus, by \Cref{lemma:decomposition} it follows that $|\E_{\mathcal{D}_{\mathbf{x}}'}[\psi(\mathbf{x}, f(\mathbf{x}))] - v | \leq \tau$ with probability at least $1 - 2\delta/q$. By the union bound, we obtain that with probability at least $1 - 2\delta$ we answer every statistical query $(\psi, \tau)$ on distribution $\mathcal{D}_{\mathbf{x}}'$ that $\mathcal{A}$ makes within the desired tolerance. As a result, by the guarantee of algorithm $\mathcal{A}$, $\algosq$ returns a hypothesis $h$ that satisfies $\pr_{\mathcal{D}_{\mathbf{x}}'}[h(\mathbf{x}) \neq f(\mathbf{x})] \leq \epsilon$, and the theorem follows from \Cref{lemma:opt_reduction}.
\end{proof}

\section{Concluding Remarks}

The main contribution of this work is twofold. First, we extended a nexus made by \cite{DBLP:journals/corr/abs-2006-04787} between evolvability and Massart learnability to a broader class of noise models. As a corollary, we established the first polynomial time learning algorithm in the presence of noise with polynomially-bounded magnitude for the fundamental class of linear threshold functions. Second, we considered a stronger oracle in which for every labeled example we additionally obtain its flipping probability. In this model, we showed that every SQ learnable concept class is also efficiently learnable under severe noise -- such as Massart and Tsybakov, strengthening a classical result by Kearns in the context of RCN. We believe that our results are of particular practical significance given that the noise models we studied throughout this work are motivated and encountered in many practical applications. 

\bibliography{./refs.bib}

\appendix

\clearpage

\section{Optimality in the Realizable Instance}
\label{appendix:realizable}

In this section we analyze whether obtaing a hypothesis $h$ such that $\err_{\mathcal{D}}(h) \leq \opt + \epsilon$ implies that $\pr_{\mathcal{D}_{\mathbf{x}}}[h(\mathbf{x}) \neq f(\mathbf{x})] \leq \epsilon'$, for some $\epsilon'$ that depends polynomially on $\epsilon$. To be more precise, we show that this is indeed the case in the Massart as well as the Tsybakov model, but as we will see it does not hold in general. 

\paragraph{Massart Model} Consider a hypothesis $h$ such that $\err_{\mathcal{D}}(h) \leq \opt + \epsilon$, for any $\epsilon > 0$. Then, given that $\eta(\mathbf{x}) \leq \gamma$, it follows that

\begin{align}
    \err_{\mathcal{D}}(h) &= \opt + \E_{\mathbf{x} \sim \mathcal{D}_{\mathbf{x}}} [(1 - 2\eta(\mathbf{x})) \mathds{1} \{ h(\mathbf{x}) \neq f(\mathbf{x}) \} ] \\
    &\geq \opt + (1 - 2\gamma) \pr_{\mathbf{x} \sim \mathcal{D}_{\mathbf{x}}}[f(\mathbf{x}) \neq h(\mathbf{x})].
\end{align}
Thus, we obtain that 

\begin{equation}
    \pr_{\mathbf{x} \sim \mathcal{D}_{\mathbf{x}}}[f(\mathbf{x}) \neq h(\mathbf{x})] \leq \frac{\epsilon}{1 - 2\gamma}.
\end{equation}

As a result, it suffices to select $\epsilon = \epsilon'(1 - 2\gamma)$ to guarantee $\epsilon'$ excess error in the underlying realizable instance. 

\paragraph{Tsybakov Model} Again, consider a hypothesis $h$ such that $\err_{\mathcal{D}}(h) \leq \opt + \epsilon$, for any $\epsilon > 0$, and fix some $t \in [0, t_0]$. Employing similar ideas to \Cref{lemma:Tsybakov-magnitude} yields that 

\begin{align}
    \err_{\mathcal{D}}(h) &= \opt + \E_{\mathbf{x} \sim \mathcal{D}_{\mathbf{x}}} [(1 - 2\eta(\mathbf{x})) \mathds{1} \{ h(\mathbf{x}) \neq f(\mathbf{x}) \} ] \\
    &\geq \opt + \int_{\mathcal{X}_{good}} (1 - 2\eta(\mathbf{x})) \mathds{1} \{ h(\mathbf{x}) \neq f(\mathbf{x}) \} \mathcal{D}_{\mathbf{x}}(\mathbf{x}) d\mathbf{x} \\
    &\geq \opt + 2t \int_{\mathcal{X}_{good}} \mathds{1} \{ h(\mathbf{x}) \neq f(\mathbf{x}) \} \mathcal{D}_{\mathbf{x}}(\mathbf{x}) d\mathbf{x},
\end{align}
where $\mathcal{X}_{good}$ is defined as in \Cref{lemma:Tsybakov-magnitude}. Moreover, given that $\pr_{\mathcal{D}_{\mathbf{x}}}[\mathbf{x} \in \mathcal{X}_{good}] \geq 1 - A t^{\frac{\alpha}{1 - \alpha}}$, we obtain that 

\begin{equation}
    \pr_{\mathbf{x} \sim \mathcal{D}_{\mathbf{x}}}[h(\mathbf{x}) \neq f(\mathbf{x})] \leq \frac{\epsilon}{2t} + A t^{\frac{\alpha}{1 - \alpha}}.
\end{equation}

Therefore, in order to get $\pr_{\mathbf{x} \sim \mathcal{D}_{\mathbf{x}}}[h(\mathbf{x}) \neq f(\mathbf{x})] \leq \epsilon'$, for any $\epsilon' > 0$, it suffices to select $\epsilon$ such that 

\begin{equation}
    \epsilon = \sup_{t \in [0, t_0]} \left\{ 2t \epsilon' - 2A t^{\frac{1}{1 - \alpha}} \right\}.
\end{equation}
In particular, it follows that 

\[
     \sup_{t \in [0, t_0]} \left\{ 2t \epsilon' - 2A t^{\frac{1}{1 - \alpha}} \right\} = \left\{\begin{array}{lr}
        2(\epsilon')^{\frac{1}{\alpha}} \left( \frac{1 - \alpha}{A} \right)^{\frac{1 - \alpha}{\alpha}} - 2A \left(\epsilon' \frac{1 - \alpha}{A} \right)^{\frac{1}{\alpha}} & \text{if } t^* \leq t_0, \\
        2 t_0 \epsilon' - 2A t_0^{\frac{1}{1 - \alpha}} & \text{if } t^* > t_0, \\
        \end{array}\right.
  \]

where 

\begin{equation}
    t^* = \left( \epsilon' \frac{1 - \alpha}{A} \right)^{\frac{1 - \alpha}{\alpha}}.
\end{equation}

On the other hand, consider the following noise function:

\begin{definition}
    A noise function $\eta(\mathbf{x})$ satisfies a $\beta$-clean condition if there exists a region $\mathcal{X}_{clean} \subseteq \mathcal{X}$ such that 
    \begin{itemize}
        \item $\pr_{\mathbf{x} \sim \mathcal{D}_{\mathbf{x}}} [\mathbf{x} \in \mathcal{X}_{clean}] \geq \beta$;
        \item $\eta(\mathbf{x}) = 0, \forall \mathbf{x} \in \mathcal{X}_{clean}$.
    \end{itemize}
\end{definition}

This noise condition allows a $1 - \beta$ fraction of the probability mass to be corrupted with noise arbitrarily close to $1/2$. 
\begin{lemma}
    The magnitude of a $\beta$-clean noise with respect to any distribution $\mathcal{D}_{\mathbf{x}}$ is upper-bounded by $1/\beta$.
\end{lemma}

\begin{proof}
    It follows that 
    
    \begin{equation}
        \mathcal{M}^{-1} = \int_{\mathcal{X}} (1 - 2\eta(\mathbf{x})) \mathcal{D}_{\mathbf{x}}(\mathbf{x}) d\mathbf{x} \geq \int_{\mathcal{X}_{clean}} \mathcal{D}_{\mathbf{x}}(\mathbf{x}) d\mathbf{x} \geq \beta.
    \end{equation}
\end{proof}

However, in this particular noise model a guarantee in the noisy distribution does not necessarily translate in the realizable instance. Indeed, assume that $\mathcal{D}_{\mathbf{x}}$ is the uniform distribution on $\mathbb{B}_2 = \{ \mathbf{x} \in \mathbb{R}^2 : ||\mathbf{x}||_2 \leq 1 \}$. We consider a partition of $\mathbb{B}_2$ into $\mathcal{X}_{clean}^r$, $\mathcal{X}_{clean}^{\ell}$, and the region $\mathbb{B}_2 \setminus (\mathcal{X}_{clean}^{r} \cup \mathcal{X}_{clean}^{\ell})$, as indicated in \Cref{fig:clean}, and we let $\pr_{\mathcal{D}_{\mathbf{x}}}[\mathbf{x} \in \mathcal{X}_{clean}^{\ell}] = \pr_{\mathcal{D}_{\mathbf{x}}}[\mathbf{x} \in \mathcal{X}_{clean}^{r}] = \frac{\beta}{2}$. In addition, we let $\eta(\mathbf{x}) = 0, \forall \mathbf{x} \in \mathcal{X}_{clean}^{r} \cup \mathcal{X}_{clean}^{\ell}$, while for the rest of the probability mass we let $\eta(\mathbf{x}) = \frac{1}{2} - \rho$, for some $\rho > 0$. 

The problem that arises is that in the limit of $\rho \to 0$, $\err_{\mathcal{D}}(h') \to \err_{\mathcal{D}}(h^*) = \opt$, for any $h'$ as in \Cref{fig:clean}. Yet, it is clear that in the realizable instance the error of $h'$ can be very far from the optimal. Nonetheless, it should be noted that a hypothesis $h$ such that $\err_{\mathcal{D}}(h) \leq \opt + \epsilon$ would classify correctly the clean data even in the presence of intense noise, a result that appears to be non-trivial and of independent interest.

\begin{figure}
    \centering
    \includegraphics[scale=0.35]{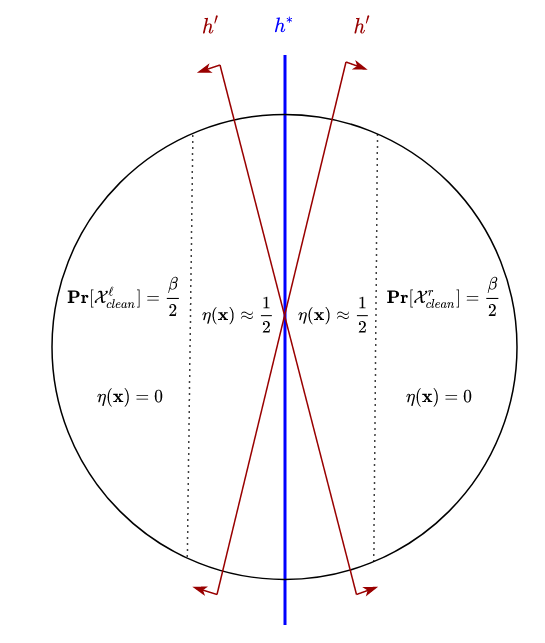}
    \caption{The geometry of our example; here $h^*$ represents the optimal classifier.}
    \label{fig:clean}
\end{figure}

\end{document}